%% file: main.tex
\documentclass{article}

\PassOptionsToPackage{square, numbers, sort, compress}{natbib}
\usepackage[preprint]{tsrml_2022}

\usepackage[utf8]{inputenc} % allow utf-8 input
\usepackage[T1]{fontenc}    % use 8-bit T1 fonts
\usepackage{caption}
\usepackage{subfig}
\usepackage{bbm}

\usepackage{dsfont}
\usepackage{adjustbox}
\usepackage{comment}
\usepackage{graphicx}
\usepackage{tabularx}
\usepackage{wrapfig}
\usepackage{floatrow}
\usepackage{amsmath}
\usepackage{hyperref}       % hyperlinks
\usepackage{url}            % simple URL typesetting
\usepackage{booktabs}       % professional-quality tables
\usepackage{amsfonts}       % blackboard math symbols      % compact symbols for 1/2, etc.
\usepackage{microtype}      % microtypography
\usepackage{xcolor}         % colors

\usepackage{float}
\floatstyle{plaintop}
\restylefloat{table}

\usepackage{amsthm}

\input{math_commands}

\input{bens_math_commands}
\usepackage{xcolor}

\author{%
  Luke Rowe$^*$\qquad Benjamin Th\'erien\thanks{Equal contribution. Firth-authorship determined by a coinflip.} \qquad Krzysztof Czarnecki \qquad Hongyang Zhang\\
   School of Computer Science \\
   University of Waterloo \\
   \texttt{\{l6rowe,btherien,k2czarne,hongyang.zhang\}@uwaterloo.ca} \\
}

\title{A Closer Look at Robustness to L-infinity and Spatial Perturbations and their Composition}

\begin{document}

\maketitle

\begin{abstract}

In adversarial machine learning, the popular $\ell_\infty$ threat model has been the focus of much previous work. While this mathematical definition of imperceptibility successfully captures an infinite set of additive image transformations that a model should be robust to, this is only a subset of all transformations which leave the semantic label of an image unchanged. Indeed, previous work also considered robustness to spatial attacks as well as other semantic transformations; however, designing defense methods against the composition of spatial and $\ell_{\infty}$ perturbations remains relatively underexplored. In the following, we improve the understanding of this seldom investigated compositional setting. We prove theoretically that no linear classifier can achieve more than trivial accuracy against a composite adversary in a simple statistical setting, illustrating its difficulty. We then investigate how state-of-the-art $\ell_{\infty}$ defenses can be adapted to this novel threat model and study their performance against compositional attacks. We find that our newly proposed TRADES$_{\text{All}}$ strategy performs the strongest of all. Analyzing its logit's Lipschitz constant for RT transformations of different sizes, we find that TRADES$_{\text{All}}$ remains stable over a wide range of RT transformations with and without $\ell_\infty$ perturbations.

\end{abstract}

% --------------------------------------------------------------------------- 
\section{Introduction}
% --------------------------------------------------------------------------- 

Despite the outstanding performance of deep neural networks\cite{yu2022coca,chen2022segmentation,zhang2022dino} on a variety of computer vision tasks, deep neural networks have been shown to be vulnerable to human-imperceptible adversarial perturbations \cite{szegedy2014intriguing, goodfellow2015steganography}. Designing algorithms that are robust to small human-imperceptible $\ell_\infty$-bounded alterations of the input has been an extensive focus of previous work \cite{madry2018towards, zhang2019trades}. While it is certainly unreasonable for a classifier to change its decision based on the addition of imperceptible $\ell_\infty$-bounded noise, this is not the only input transformation we wish to be robust to. Many spatial transformations, such as bounded rotation/translations (RTs), leave an image's label unchanged, but are ill-defined by an $\ell_\infty$-threat model (see Figure~\ref{fig:compositional-adv-example}). Yet, any classifier deemed robust should not be any more vulnerable to perturbations applied to RT transformed images seen in Figure~\ref{fig:compositional-adv-example} (row 3) than to natural images (row 1). However, current defenses, designed for $\ell_\infty$ robustness fail under this compositional setting (see Table~\ref{table:cifar10}), suggesting that our models, at least for image classification, are less robust than we thought. To build truly robust models, we must design training protocols to account for such situations.

While many prior works have considered robustness under adversarial settings that differ from the standard $\ell_\infty$ setting, most works either consider robustness under a single perturbation type \cite{engstrom2017suffice, xiao2018spatial, kanbak2018geometricrobustness, balunovic2019certifygeometric, engstrom2019exploring, li2021tss} or by selecting a perturbation from a fixed set (\textit{i.e.,} the union) of perturbation types \cite{ tramer2019multiple, maini2020union,li2021tss, laidlaw2021perceptual, maini2022categorization}. However, relatively few consider robustness to the composition of multiple perturbation types \cite{tramer2019multiple, li2021tss, tsai2022compositional}. Realistically, an adversary is not restricted to selecting a perturbation from one threat model but may choose to compose perturbations from multiple threat models (see Figure~\ref{fig:compositional-adv-example}). Moreover, our theoretical analysis shows that defending against an adversary who can compose $\epsilon$-bounded $\ell_{\infty}$ perturbations and RT transformations is challenging even in a simple statistical setting. This theoretical result highlights the need to explore how we can build truly robust models in this well-motivated compositional setting. The main contributions of this work are three-fold:

\begin{figure}[t]
    \centering
    \includegraphics[width=0.6\linewidth]{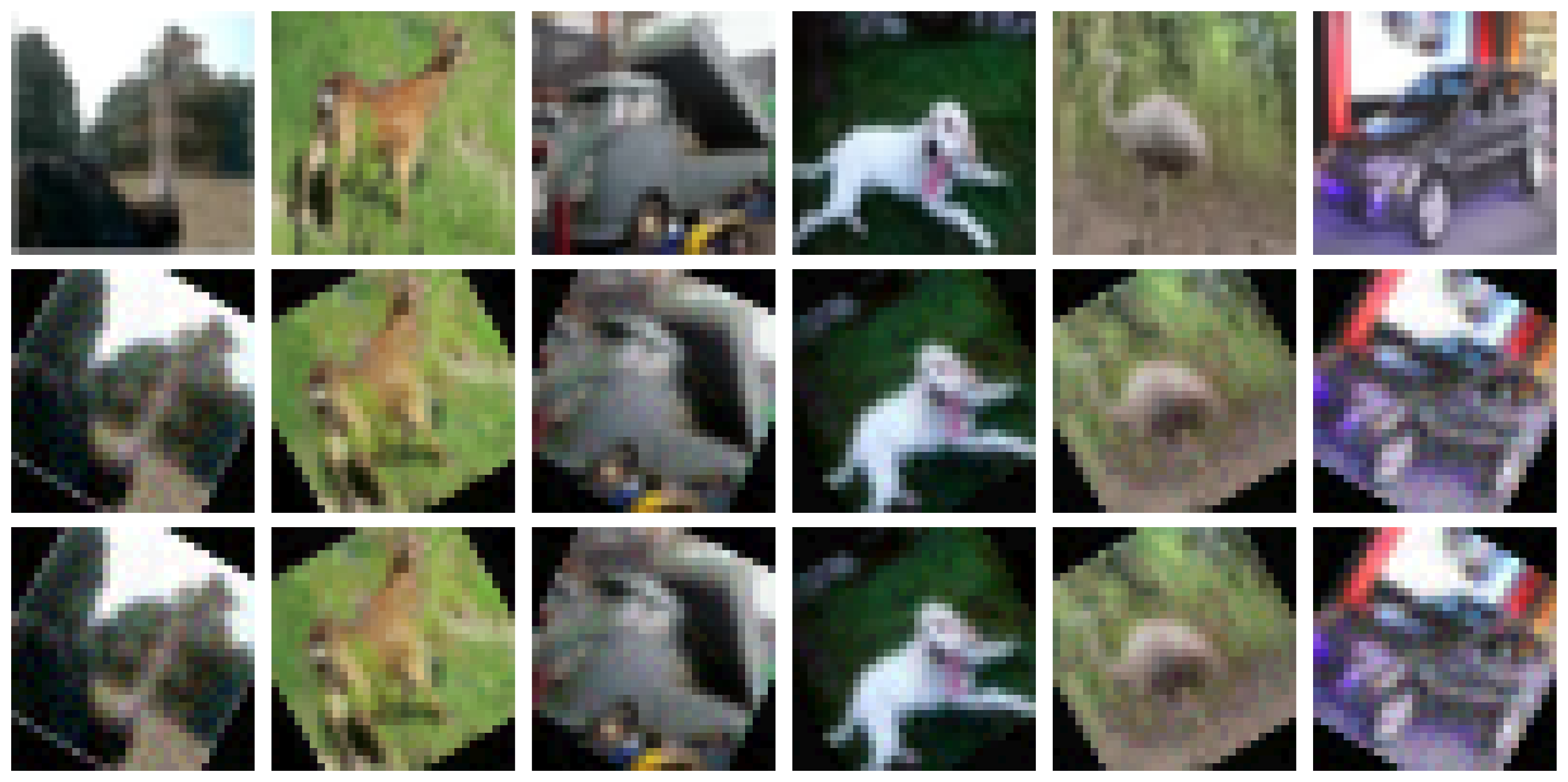}
    \caption{\textbf{Adversarial images obtained by the AAA $\circ$ RT attack}. The first row shows clean images. The second row shows the same images with applied adversarial RT transformations, and the third row shows the same RT images as above, but with imperceptible $\ell_\infty$ perturbations applied via AAA. }
    \label{fig:compositional-adv-example}
\end{figure}

\begin{itemize}
\item We show theoretically that no linear classifier can attain non-trivial compositional robustness in a simple, yet realistic, statistical setting. 
\item We train a family of empirical defenses constructed from TRADES \cite{zhang2019trades} and analyze their performance under a compositional adversary. 
\item We propose $\text{TRADES}_{\text{All}}$, a new training protocol for defending against $\ell_\infty \circ \text{RT}$ adversaries, show that it attains the best performance of all the defenses trained, and discover that its logits are more stable than our other robust models, shedding light on its strong performance.% analyze reasons for these performance improvements.
\end{itemize}

% --------------------------------------------------------------------------- 
\section{Related work}
% --------------------------------------------------------------------------- 
\label{sec:relatedwork}

Many existing works consider adversarial robustness to single perturbation types, which include robustness to $\ell_{p}$ perturbations \cite{szegedy2014intriguing, goodfellow2015steganography, madry2018towards, zhang2019trades, zhang2022causal} as well as robustness to spatial transformations of the input \cite{xiao2018spatial, balunovic2019certifygeometric, engstrom2017suffice, li2021tss, engstrom2019exploring, kanbak2018geometricrobustness}. Compared with single perturbation type robustness, relatively few consider the problem of attaining robustness to the composition of multiple perturbation types \cite{tramer2019multiple, li2021tss, tsai2022compositional, mao2021compositeadversarialattacks}. Tram{\`{e}}r and Boneh~\cite{tramer2019multiple} first identified the compositional setting and studied the composition of multiple $\ell_{p}$ perturbations as well as $\ell_{\infty}$ and RT perturbations; however, they consider affine combinations of multiple perturbations, which unreasonably constrains the power of the compositional adversary. Li et al. \cite{li2021tss} designed methods to attain certified robustness to the composition of various semantic transformations of the input, and Tsai et al. \cite{tsai2022compositional} designed a generalized form of adversarial training for compositional semantic perturbations. Mao et al. \cite{mao2021compositeadversarialattacks} designed a composite adversarial attack that composes the search space of multiple base attackers. Our work is most closely related to \cite{tramer2019multiple}; however, we differ from \cite{tramer2019multiple} by considering the addition of $\ell_{\infty}$ perturbations and RT transformations rather than an affine combination of such perturbations in our analysis, so as to not unreasonably limit the strength of the compositional adversary. Figure \ref{fig:compositional-adv-example} motivates this treatment, as the second and third row of images are indistinguishable to humans.

\section{Preliminaries}
\label{sec:preliminaries}

In this work, we consider a compositional threat model consisting of the composition of $\epsilon$-bounded $\ell_{\infty}$ perturbations and bounded RT transformations. For the $\ell_{\infty}$ threat model, we consider an adversary who can perturb an image $\vx$ with $\epsilon$-bounded $\ell_{\infty}$ noise. That is, the adversarial reachable region $\mathcal{A}^{\ell_{\infty}}(\vx)$ under the $\ell_{\infty}$-threat model is defined by:
\begin{align}
    \mathcal{A}^{\ell_{\infty}}(\vx) = \mathbb{B}_{\infty}(\vx, \epsilon) := \{ \vx + \boldsymbol{\Delta} ; || \boldsymbol{\Delta} ||_{\infty} \leq \epsilon \}.
\end{align}
For the RT threat model, we consider an adversary who can apply a bounded rotation $\theta$ followed by bounded horizontal and vertical translations $\delta_x, \delta_y$ to $\vx$. Concretely, the adversarial reachable region under the RT-threat model is defined by:
\begin{align}
    \mathcal{A}^{\text{RT}}(\vx) = \{\mathcal{T}(\vx; \theta, \delta_x, \delta_y) ; |\theta| \leq \theta^{\text{max}}, |\delta_x| \leq \delta_x^{\text{max}}, |\delta_y| \leq \delta_y^{\text{max}} \},
\end{align}
where $\mathcal{T}(\cdot \text{ }; \theta, \delta_x, \delta_y)$ is the affine transformation function with rotation $\theta$ and horizontal/vertical translations $\delta_x, \delta_y$, which implicitly warps the image via an interpolation algorithm (our experiments utilize bilinear interpolation). For the compositional threat model, the adversarial reachable region is naturally defined by:
\begin{align}
   \mathcal{A}^{\ell_{\infty} \circ \text{ RT}}(\vx) = \{\mathbb{B}_{\infty} (\mathcal{T}(\vx; \theta, \delta_x, \delta_y), \epsilon) ; |\theta| \leq \theta^{\text{max}}, |\delta_x| \leq \delta_x^{\text{max}}, |\delta_y| \leq \delta_y^{\text{max}}  \}. 
\end{align}
That is, $\mathcal{A}^{\ell_{\infty} \circ \text{ RT}}(\vx)$ is defined as as the set of $\epsilon$-bounded $\ell_{\infty}$ balls around all valid affine transformations of the image $\vx$. To compare with our compositional setting, we also consider the union threat model consisting of the union of $\epsilon$-bounded $\ell_{\infty}$ perturbations and bounded RT transformations \cite{tramer2019multiple}. In this case, the adversarial reachable region under the $\ell_{\infty} \cup \text{ RT}$-threat model is defined as $\mathcal{A}^{\ell_{\infty} \cup \text{ RT}}(\vx) =  \mathcal{A}^{\ell_{\infty}}(\vx) \cup \mathcal{A}^{\text{RT}}(\vx)$.

% --------------------------------------------------------------------------- 
\section{On the Difficulty of Attaining Compositional Robustness with Linear Classifiers}
% --------------------------------------------------------------------------- 

In this section, we theoretically demonstrate the difficulty of defending against an $\ell_{\infty} \circ \text{RT}$ compositional adversary with a linear classifier on a simple statistical setting. 

\subsection{Statistical Setting}

To theoretically analyze the compositional adversarial setting, we use the statistical distribution proposed in \cite{tsipras2018robustness}. Namely, we study a binary classification problem with $d$-dimensional input features, in which the first feature $X_0$ is strongly correlated with the output label $y$ with probability $p$, and the remaining features are weakly correlated with $y$. The distribution can be written as follows:
\begin{align}\label{eq:d}
\hspace{-4px}Y \stackrel{\text{u.a.r.}}{\sim} \{-1, 1\}, \; X_0 | Y = y := \begin{cases}
			y, & \text{w.p. } p;\\
            -y, & \text{w.p. } 1-p,
		 \end{cases} \; X_t | Y = y \sim \mathcal{N}\left(y\eta, 1\right),\;\; 1 \leq t\leq d-1, %=1,2,...,d. 
\end{align}
where $\eta = \Theta(\frac{1}{\sqrt{d}})$ and $p \geq 0.5$. We assume that an $\ell_{\infty}$ adversary has budget $\epsilon = 2 \eta$, similar to \cite{tsipras2018robustness}. Moreover, we define an RT transformation as it is defined in \cite{tramer2019multiple}. Concretely, an RT transformation is defined as a swap between the strongly correlated feature $X_0$ and a weakly correlated feature $X_t$, $1 \leq t \leq d-1$. To constrain the RT transformation, we assume that an RT adversary can swap $X_0$ with at most $N$ positions on the input signal. If we assume the input features $X_0$, \dots, $X_{d-1}$ lie on a 2-dimensional grid, then this definition of an RT transformation serves as a realistic abstraction of applying an RT transformation to an image using \textit{nearest} interpolation and rotating about the image's center. Namely, since the distribution over the last $d-1$ features is permutation invariant, then the only power of an RT transformation is to move the strongly correlated feature, where $N$ defines the number of reachable pixels that the strongly correlated feature can be mapped to via an RT transformation. For example, when considering only translations, we have $N = (2\delta_x^{\text{max}} + 1)(2 \delta_y^{\text{max}} + 1)$. We now state a theorem that establishes the difficulty of defending against an $\ell_{\infty} \circ \text{RT}$ compositional adversary with a linear classifier.
\begin{theorem}[A linear classifier cannot attain nontrivial $\ell_{\infty} \circ \text{RT}$ robustness]
\label{theorem: main}
Given data distribution $\mathcal{D}$ where $p \geq \frac{1}{2}$ , $\eta \geq \frac{1}{\sqrt{d}}$, and $d \geq 24$, no linear classifier $f:\Rb^{d} \rightarrow \{ -1, 1 \}$, where $f(\vx) = \sign(\vw^T \vx)$, can obtain robust accuracy $> 0.5$ under the $\ell_{\infty}\circ \text{RT}$ threat model with $\ell_{\infty}$ budget $\epsilon=2\eta$ and RT budget $N = \frac{d}{8}$.
\end{theorem}
This theorem shows that under reasonable constraints on the compositional adversary, a linear classifier can perform no better than random, even in the infinite data limit. We note by contrast that a linear classifier can attain $> 0.99$ natural accuracy in this statistical setting; \textit{e.g.}, see \cite{tsipras2018robustness}. This result distinguishes itself from Theorem 4 in \cite{tramer2019multiple}, in that \cite{tramer2019multiple} show that an adversary that composes $\ell_{\infty}$ and RT perturbations yields a stronger attack than a union adversary, whereas we show that a linear classifier cannot have nontrivial robustness against a compositional adversary under this statistical setting. We emphasize that although no linear classifier can attain nontrivial robustness on this statistical setting, networks with sufficient depth and capacity may be able to attain nontrivial robustness (see $\text{TRADES}_{\text{All}}$ in Tables \ref{table:mnist} and \ref{table:cifar10}). Nevertheless, this result highlights the difficulty of attaining compositional robustness in this setting. We next explore how we can design robust models in this well-motivated compositional setting.

% --------------------------------------------------------------------------- 
\section{Experiments}
\label{sec:methodology}
% --------------------------------------------------------------------------- 

\subsection{Proposed defense methods}
\label{sec:defensemethods}
To explore the space of compositional adversarial examples and the compositional threat model, we train a family of empirical defenses constructed from TRADES \cite{zhang2019trades} and evaluate these defenses in a white-box setting. We choose a white-box setting to assess the full adversarial strength of these compositional adversarial examples. Below, we have a general form for the TRADES objective:
\begin{align}
    \min_f \mathbb{E} \Bigg\{ \underbrace{\mathcal{L}(f(\vx), Y)}_{\text{Natural Accuracy}} + \beta \underbrace{\max_{\vx' \in \mathcal{A}(\vx)} \mathcal{L}(f(\vx), f(\vx'))}_{\text{Robustness under } \mathcal{A}(\vx)} \Bigg\}.\label{eq:trades}
\end{align}
We train a family of TRADES models under the various threat models discussed in Section  \ref{sec:preliminaries}. Concretely, we train the following family of TRADES defense methods: $\text{TRADES}_{\ell_{\infty}}$, $\text{TRADES}_{\text{RT}}$, $\text{TRADES}_{\ell_{\infty} \cup \text{ RT}}$, and $\text{TRADES}_{\ell_{\infty} \circ \text{ RT}}$, where the subscript indicates the threat model being considered during training. To train TRADES models under these new threat models, we require a way to efficiently solve the inner optimization problem in the TRADES objective for the new corresponding definitions of $\mathcal{A}(\vx)$. In the $\ell_{\infty}$ case, we perform Projected Gradient Descent (PGD) for a small number of steps, as is typically done \cite{madry2018towards}. For the RT-threat model, we perform a \textbf{Worst-of-10} search: we sample 10 random valid affine transformations and select the affine transformed image that attains the highest loss, as is done in \cite{engstrom2019exploring}. For the robustness loss function, we use the KL-divergence between the logits of the natural image and the logits of the transformed image. For the union setting, we use an existing approach called the \textbf{Max Strategy} \cite{tramer2019multiple}, in which we compute an $\ell_{\infty}$ perturbation using PGD and an RT perturbation using Worst-of-10, and select the perturbation that attains the maximum KL-divergence loss. For the compositional setting, we propose the \textbf{Worst-on-Worst} strategy, whereby we first compute an RT adversarial example using Worst-of-10, and then we perform PGD on the worst RT-perturbed image. Worst-on-Worst implicitly assumes that the ``worst'' adversarial image from Worst-of-10 will produce the ``worst'' compositional adversarial example.

\begin{table}[t] 
    \centering
    \caption{\textbf{MNIST results for different defense methods.} Columns correspond to robust accuracy under different perturbation types, whereas rows correspond to different defense models. All RT attacks utilize our grid search strategy. PGD attacks use $40$ iterations on MNIST. The best performing entry under a given attack is \textbf{bolded}, while the second best is \underline{underlined}.} 
    \fontsize{7}{8.5}\selectfont 
    \begin{tabularx}{0.975\linewidth}{lc|rr|rr|rrrr} 
\toprule
Defense \text{\textbackslash} Attack &\multicolumn{1}{c|}{$\beta$}&AAA $\circ$ RT&PGD $\circ$ RT&AAA $\cup$ RT&PGD $\cup$ RT&AAA&PGD&RT&Natural \\\midrule 
$\text{TRADES}_{\text{All}}$ & $1.0$ & $49.47$ & $68.60$ & $90.28$ & $\underline{92.49}$ & $92.54$ & $95.60$ & $93.93$ & $99.34$ \\
$\text{TRADES}_{\ell_{\infty} \circ \text{ RT}}$& $1.0$ & $55.61$ & $71.37$ & $89.28$ & $91.51$ & $91.66$ & $95.10$ & $93.06$ & $99.03$ \\
$\text{TRADES}_{\ell_{\infty} \cup \text{ RT}}$ & $1.0$ & $21.97$ & $49.16$ & $89.67$ & $92.31$ & $91.32$ & $95.20$ & $94.08$ & $99.47$ \\
$\text{TRADES}_{\ell_{\infty}}$ & $1.0$ & $0.02$ & $00.11$ & $00.43$ & $00.43$ & $\underline{92.99}$ & $95.88$ & $00.57$ & $\underline{99.52}$\\
$\text{TRADES}_{\text{RT}}$ & $1.0$ & $0.00$ & $0.07$ & $0.00$ & $0.18$ & $0.00$ & $0.18$ & $96.63$ & $\boldsymbol{99.64}$ \\[0.8mm]

$\text{TRADES}_{\text{All}}$ & $3.0$ & $58.68$ & $72.70$ & $90.59$ & $92.35$ & $92.33$ & $95.00$ & $94.05$ & $98.94$ \\
$\text{TRADES}_{\ell_{\infty} \circ \text{ RT}}$& $3.0$ & $\underline{59.28}$ & $\underline{73.74}$ & $88.93$ & $91.07$ & $91.03$ & $94.01$ & $93.09$ & $98.65$ \\
$\text{TRADES}_{\ell_{\infty} \cup \text{ RT}}$ & $3.0$ & $39.72$ & $66.49$ & $\underline{91.00}$ & $\boldsymbol{93.43}$ & $92.21$ & $95.62$ & $95.00$ & $99.28$ \\
$\text{TRADES}_{\ell_{\infty}}$ & $3.0$ & $0.04$ & $0.11$ & $0.48$ & $0.48$ & $\boldsymbol{93.83}$&$\boldsymbol{96.47}$&$0.64$&$99.35$ \\
$\text{TRADES}_{\text{RT}}$ & $3.0$ & $0.00$&$0.01$&$0.00$&$0.01$&$0.00$&$0.17$&$\boldsymbol{97.54}$&$99.49$ \\[0.8mm]

$\text{TRADES}_{\text{All}}$ & $6.0$ & $\boldsymbol{61.74}$ & $\boldsymbol{75.38}$ & $90.33$ & $92.29$ & $92.19$ & $95.25$ & $93.60$	& $98.73$ \\
$\text{TRADES}_{\ell_{\infty} \circ \text{ RT}}$& $6.0$ & $58.21$ & $72.54$ & $88.25$ & $90.27$ & $90.62$ & $93.75$ & $92.21$ & $98.23$ \\
$\text{TRADES}_{\ell_{\infty} \cup \text{ RT}}$ & $6.0$ & $44.52$ & $69.97$ & $\boldsymbol{91.25}$ & $\boldsymbol{93.43}$ & $92.49$ & $95.46$ & $94.99$ & $99.22$ \\
$^\dagger\text{TRADES}_{\ell_{\infty}}$ & $6.0$ & $0.01$ & $0.07$ & $0.48$ & $0.48$ & $92.73$ & $\underline{96.07}$ & $0.58$ & $99.48$ \\
$\text{TRADES}_{\text{RT}}$ & $6.0$ & $0.00$ & $0.00$ & $0.00$ & $0.07$ & $0.00$ & $0.07$ & $\underline{97.48}$ & $99.42$ \\[0.8mm]

Natural & - & $0.00$ & $0.00$ & $0.00$ & $2.13$ & $0.00$ & $2.18$ & $0.19$ & $99.18$ \\

\bottomrule
\end{tabularx}
    \begin{flushleft}
    \tiny
    \vspace{-4pt}
    \hspace{14pt}$\dagger$ model checkpoint from \cite{zhang2019trades}.\\
\end{flushleft}
\vspace{-7pt}
\label{table:mnist}
\end{table}

\subsubsection{\texorpdfstring{$\text{TRADES}_{\text{All}}$}{l infinity}}
\label{sec:tradesall}

We observe that the family of TRADES models proposed in Section \ref{sec:defensemethods} each train a model exclusively on adversarial images tailored to their respective threat models. However, this may not strike a favourable balance in performance between the different threat models at evaluation time. To address this issue, we propose $\text{TRADES}_{\text{All}}$, whereby adversarial training alternates between $\ell_{\infty}$ adversarial examples, RT adversarial examples and compositional adversarial examples. Concretely, given training image $\vx$, $\text{TRADES}_{\text{All}}$ selects uniformly at random between a corresponding $\ell_{\infty}$ adversarial example, RT adversarial example, and a compositional adversarial example for $\vx$ when solving the inner maximization problem. The aim of this defense is to strike the right balance between all these perturbation types, without over-optimizing on a single perturbation type. 

\subsection{Attack methods background}
\label{attacks}

We now describe the white-box attack algorithms used to evaluate our family of TRADES defenses. For the $\ell_{\infty}$-based attacks, we evaluate on Adaptive AutoAttack \cite{liu2022aaa}, or AAA, which is a recently published state-of-the-art adaptive white box attack. The exact algorithmic details of AAA can be found in \cite{liu2022aaa}. For the RT based-attack, we perform a simple grid-search on the 3 parameters $\theta, \delta_x, \delta_y$ that define the affine transformation. The grid search involves evenly-spaced values for each parameter: $\theta$ has 12 values, and $\delta_x, \delta_y$ each have 5 values. Our complete attack suite is as follows: AAA, PGD, RT (Grid Search), AAA $\cup$ RT, PGD $\cup$ RT, AAA $\circ$ RT, and PGD $\circ$ RT. For crafting perturbations of multiple types, the union attacks use the Max Strategy and the composition attacks use the Worst-of-Worst strategy.

% --------------------------------------------------------------------------- 
% \section{Empirical Evaluation of the TRADES defenses}
% \label{sec:eeandresults}
% --------------------------------------------------------------------------- 

% when referring to robust accuracy under a certain attacks we always refer to the attacks using AAA. 
%and $\ell_\infty$ threat model of $0.3$ and $0.031$, respectively following previous work \cite{zhang2019trades}. 

\begin{table}[t]
    \centering
    \caption{\textbf{CIFAR-10 results for different defense methods.} Columns correspond to robust accuracy under different perturbation types, while rows correspond to different defense models. All RT attacks utilize our grid search strategy. PGD attacks use $20$ iterations on CIFAR-10. The best performing entry under a given attack is \textbf{bolded}, while the second best is \underline{underlined}.} 
    \fontsize{7}{8.5}\selectfont 
    \begin{tabularx}{0.975\linewidth}{lc|rr|rr|rrrr} 
\toprule
Defense \text{\textbackslash} Attack &\multicolumn{1}{c|}{$\beta$}&AAA $\circ$ RT&PGD $\circ$ RT&AAA $\cup$ RT&PGD $\cup$ RT&AAA&PGD&RT&Natural \\\midrule
$\text{TRADES}_{\text{All}}$ & $3.0$ & $\underline{33.31}$&$\underline{37.60}$&$48.28$&$52.30$&$48.53$&$52.66$&$77.87$&$85.49$ \\
$\text{TRADES}_{\ell_{\infty} \circ \text{ RT}}$& $3.0$ & $25.79$&$30.12$&$40.02$&$44.24$&$40.32$&$44.66$&$73.63$&$83.81$ \\
$\text{TRADES}_{\ell_{\infty} \cup \text{ RT}}$ & $3.0$ & $7.01$&$10.47$&$46.46$&$49.79$&$47.91$&$51.48$&$76.14$&$86.42$ \\
$\text{TRADES}_{\ell_{\infty}}$ & $3.0$ & $2.51$&$3.00$&$9.53$&$9.85$&$\underline{51.20}$&$54.13$&$15.50$&$86.36$ \\
$\text{TRADES}_{\text{RT}}$ & $3.0$ & $0.00$&$0.00$&$0.00$&$0.00$&$0.00$&$0.00$&$\boldsymbol{82.31}$&$\boldsymbol{95.46}$ \\[0.8mm]

$\text{TRADES}_{\text{All}}$ & $6.0$ & $\boldsymbol{35.33}$ &$\boldsymbol{40.24}$&$\underline{49.17}$&$\underline{53.53}$&$49.45$&$54.00$&$76.01$&$83.65$ \\
$\text{TRADES}_{\ell_{\infty} \circ \text{ RT}}$& $6.0$ & $19.40$&$23.69$&$38.62$&$42.81$&$39.84$&$44.47$&$68.01$&$82.77$ \\
$\text{TRADES}_{\ell_{\infty} \cup \text{ RT}}$ & $6.0$ & $14.25$&$20.43$&$\boldsymbol{49.86}$&$\boldsymbol{53.65}$&$50.66$&$\underline{54.77}$&$77.39$&$84.72$ \\
$^\dagger\text{TRADES}_{\ell_{\infty}}$ & $6.0$ & $3.52$&$4.49$&$10.55$&$11.19$&$\boldsymbol{53.01}$&$\boldsymbol{56.63}$&$18.51$&$84.92$ \\
$\text{TRADES}_{\text{RT}}$ & $6.0$ & $0.00$&$0.00$&$0.00$&$0.01$&$0.00$&$0.01$&$\underline{81.98}$&$\underline{94.91}$ \\[0.8mm]

Natural & - & $0.00$&$0.00$&$0.00$&$0.00$&$0.00$&$0.00$&$25.94$&$93.68$ \\

\bottomrule
\end{tabularx}
\begin{flushleft}
\tiny
\vspace{-4pt}
\hspace{14pt}$\dagger$ model checkpoint from \cite{zhang2019trades}.\\
\end{flushleft}
\vspace{-7pt}
\label{table:cifar10}
\end{table}

% --------------------------------------------------------------------------- 
\subsection{Results}
% ---------------------------------------------------------------------------
In this section, we present our empirical evaluation of the proposed family of TRADES defense models trained and evaluated on MNIST and CIFAR-10 in the threat models introduced above. We provide training details and the computational complexity of the proposed methods in Section~\ref{sec:app-complex} of the Appendix.

The results of our empirical evaluations are reported in Tables \ref{table:mnist} and \ref{table:cifar10}. On both datasets, we observe that composition attacks are empirically stronger than union attacks. This finding aligns with the theoretical result from \cite{tramer2019multiple}, which shows that the compositional setting is harder than the union in a simple statistical setting. Furthermore, $\ell_\infty$ attacks are empirically stronger than their RT counterparts against all non-specialized models. Analyzing the effect of $\beta$, we note that in general $\ell_\infty$ robustness benefits from higher values, while robustness to RT attacks and natural accuracy generally decreases as $\beta$ is increased. Moreover, RT robust models obtain the strongest natural accuracy of all classifiers on all datasets, showing that RT robustness and natural accuracy complement each other on these natural image distributions. 

The sixth row of Table \ref{table:cifar10} showcases the strong overall performance of the $\text{TRADES}_{\text{All}}$ strategy at $\beta=6.0$, which performs best on composition attacks and when accounting for all settings together on CIFAR-10. Its improved performance against composition attacks when compared to the model trained exclusively on compositions suggests that alternating training schemes can be beneficial in this setting. Moreover, the mediocre performance of the union trained models in these settings demonstrates that training on the union is insufficient to defend against a composite adversary. 

On MNIST (Table~\ref{table:mnist}), we observe similar trends to CIFAR-10. The All-trained models shine against the composition attack and $\text{TRADES}_{\text{All}}$ at $\beta=6.0$ performs strongest of all. Specialized models perform best on the union, $\ell_\infty$, and RT, while TRADES$_\text{RT}$ at $\beta=1.0$ performs strongest on natural images. 

% --------------------------------------------------------------------------- 
\vspace{-5pt}
\subsection{Analyzing the strong performance of TRADES\texorpdfstring{$_{\text{All}}$}{all}}
\vspace{-5pt}
\label{sec:logits}
% --------------------------------------------------------------------------- 
In this section, we investigate the strong performance of our $\text{TRADES}_{\text{All}}$ classifier by inspecting its logit's Lipschitz constant for different RT transformation sizes. Figure~\ref{fig:logits-linf-rt-no-nat} plots the median $\|f(\Ic)-f(\Ic')\|_2$ for a batch of 128 CIFAR-10 images on the y-axis, where $f$ is one of our trained classfiers, $\Ic$ is the natural image, and $\Ic'=\text{PGD}_{10}(\mathcal{T}(\Ic; \theta, \delta_x, \delta_y))$ is the compositional adversarial example. The transformation strength is on the $x$-axis, measured by summing $|\theta|, |\delta_x|,$ and $ |\delta_y|$. The RT transformations are sampled in the same way as our grid-search (Section~\ref{sec:methodology}), therefore applying a PGD-10 perturbation to these images can be seen as a weak composite adversary. We note that $\text{TRADES}_{\text{All}}$ preserves similar median logit stability across a range of $\ell_\infty$-perturbed RT transformed images, suggesting that logit stability is desirable for achieving strong robustness to the composition. By stability we mean attaining a small Lipschitz constant across the valid range of RT transformations.

When assessing logit stability for different RT transformation sizes \textit{without} applying $\ell_\infty$ perturbations (see Figure~\ref{fig:logits-rt}), we note that the strongest natural classifiers (RT and Natural) are nearly invariant to RT transformations for $|\theta| + |\delta_x| + |\delta_y| \leq 20$, yet they achieve nearly $0$ accuracy against the composition. Interestingly,  $\text{TRADES}_{\text{All}}$ is the most stable of all robust models (excluding RT) to RT transformations, with a median Lipschitz constant of $\sim$$0.02$ across the board (close to RT and natural classifiers), yet it is the most robust to the composite adversary. This suggests that strong stability to RT transformations is needed for robustness on the composition, but that too much may lead to degraded accuracy against $\ell_\infty$ adversaries. In lieu of this discovery, we can see our $\text{TRADES}_{\text{All}}$ training protocol as being designed to make exactly these tradeoffs: its RT TRADES training encourages logit invariance to RT transformations, while $\ell_\infty$ and $\ell_\infty \circ \text{RT}$ training encourage it to also be stable in those settings.

\subsection{Takeaways} 
Our empirical study has three main takeaways: composite attacks are stronger than $\ell_\infty$, RT, and $\ell_\infty \cup$ RT; complex alternating schemes may be needed to train defenses robust to $\ell_\infty \circ$ RT; and robustness tradeoffs exist between specialized and general models. As the results show, all defense methods on both datasets show significant reductions in robust accuracy when defending against composition attacks, while these images appear no different than their RT counterparts (Figure~\ref{fig:compositional-adv-example}). This demonstrates that obtaining truly robust models may be even more difficult than was previously thought. While the problem is certainly very difficult, alternating training schemes seem to help bridge the gap between performance against $\ell_\infty \circ$ RT adversaries and $\ell_\infty$ adversaries. Our $\text{TRADES}_{\text{All}}$ training strategy garners the most robustness in this setting, while sacrificing relatively little in terms of $\ell_\infty$ or natural accuracy, compared to its $\ell_\infty$ trained counterparts.

\vspace{-5pt}
\section{Conclusion}
\vspace{-5pt}
\label{sec:conclusion}
% --------------------------------------------------------------------------- 
Defending against compositional threat models is a difficult but necessary task. 
% While robust models should certainly not need to defend against every possible compositions of adversaries, for instance $\ell_\infty \circ \ell_1$ or $\ell_\infty \circ \ell_2$ make little sense, compositions that leave images imperceptibly altered (Figure~\ref{fig:compositional-adv-example}) but allow for stronger attacks must be considered. 
Our contributions take a step towards this goal by highlighting the difficulty of the composite setting for a linear classifier; proposing $\text{TRADES}_{\text{All}}$, a new training strategy, which overcomes this difficulty; and empirically benchmarking its performance relative to other relevant baselines.
Our experiments show that alternating training schemes are critical for striking a balance between the different threat models. However, even our best performing method, TRADES$_\text{All}$, does not match the robust performance of specialized models in the different settings considered (except $\ell_\infty \circ \text{RT}$). These results highlight the need for future research under this threat model to improve theoretical understanding and build stronger empirical defenses.%, which we leave for future work.

% \begin{ack}
% We would like to acknowledge support for this project
% from the University of Waterloo regarding their generous MS Funding.

% Do {\bf not} include this section in the anonymized submission, only in the final paper. You can use the \texttt{ack} environment provided in the style file to autmoatically hide this section in the anonymized submission.
% \end{ack}

\clearpage

{\small
\bibliographystyle{abbrvnat}
\bibliography{tsrml_2022.bib}
}

\appendix

\clearpage
% ---------------------------------------------------------------------------
\section{Training details \& training time of TRADES models}%$\text{TRADES}_{\text{All}}$}
\label{sec:app-complex}
% ---------------------------------------------------------------------------

\begin{table}[ht]
\scriptsize
    \centering
    \caption{\textbf{Training times in HH:MM format for different defense models.} All models were trained on a single NVIDIA V100 GPU with the same hyperparameters. The only difference between each entry is their solution to the inner-maximization problem, decribed in Section \ref{sec:defensemethods}. We note that reported times include one AAA attack performed after training which introduces some variance of approximately $\pm 30$ mins for CIFAR-10 and $\pm 7$ mins for MNIST.}
    \begin{tabular}{l|c|c|c|c|c|c}
    \toprule
         Dataset \text{\textbackslash} Defense &Natural&$\text{TRADES}_{\ell_\infty}$& $\text{TRADES}_{\text{All}}$ &$\text{TRADES}_{\text{RT}}$&$\text{TRADES}_{\ell_{\infty} \circ \text{ RT}}$ &$\text{TRADES}_{\ell_{\infty} \cup \text{ RT}}$  \\ %\hline %\\[-2.2mm]
         \midrule
         \textbf{MNIST} & $00:16$& $01:52$ & $05:49$& $07:05$& $08:24$ & $08:32$ \\%\bottomrule
         \textbf{CIFAR-10}& $03:09$ & $35:25$ & $30:21$& $21:12$& $47:28$ & $49:36$ \\
         \bottomrule
    \end{tabular}
    \label{table:runtime}
\end{table}

\subsection{Training setup}
We take a number of steps to improve the reproducibility of our results and comparability to prior work. Firstly, we use the WRN-34-10 model for the CIFAR-10 experiments and the SmallCNN model for the MNIST experiments, as is done in TRADES, and to train our family of TRADES defenses we utilize the default hyperparameters included in the author's GitHub repository. \footnote{\label{foot:1}see "train\_trades\_cifar10.py" and "train\_trades\_mnist.py" at \url{https://github.com/yaodongyu/TRADES}} Secondly, we use the same seed for each experiment to obtain the same random weight initializations and dataset shuffles. Thirdly, to attack our defense models, we use the default code provided by Liu et al. \cite{liu2022aaa} for their AAA implementation and we borrow more code from the TRADES repository for PGD attacks, where PGD attacks on CIFAR-10 are run for $20$ iterations with $\epsilon = 0.031$, while attacks on MNIST are run for $40$ iterations with $\epsilon = 0.3$ (the same $\epsilon$ values are used for AAA). Lastly, in the RT-threat model, we set $\theta^{\text{max}} = 30, \delta_x^{\text{max}} = 3\text{px}, \delta_y^{\text{max}} = 3\text{px}$, which is consistent with prior work \cite{engstrom2019exploring}. 

\subsection{Training time}
We note that our $\text{TRADES}_{\text{All}}$ model achieves the best overall accuracy while being more computationally efficient than TRADES models trained on $\text{RT},\ell_{\infty} \circ \text{ RT},$ or $\ell_{\infty} \cup \text{ RT}$ settings (see CIFAR-10 table.~\ref{table:runtime}). This is because the All strategy only requires computing $\ell_\infty$ and $\text{RT}$ perturbations for $\frac{2}{3}$ of the images every epoch. 

\clearpage
% --------------------------------------------------------------------------- 
\section{Logit Plots}
\label{sec:app-logits}
% --------------------------------------------------------------------------- 

\begin{figure}[ht]
    \centering
    \includegraphics[width=\linewidth]{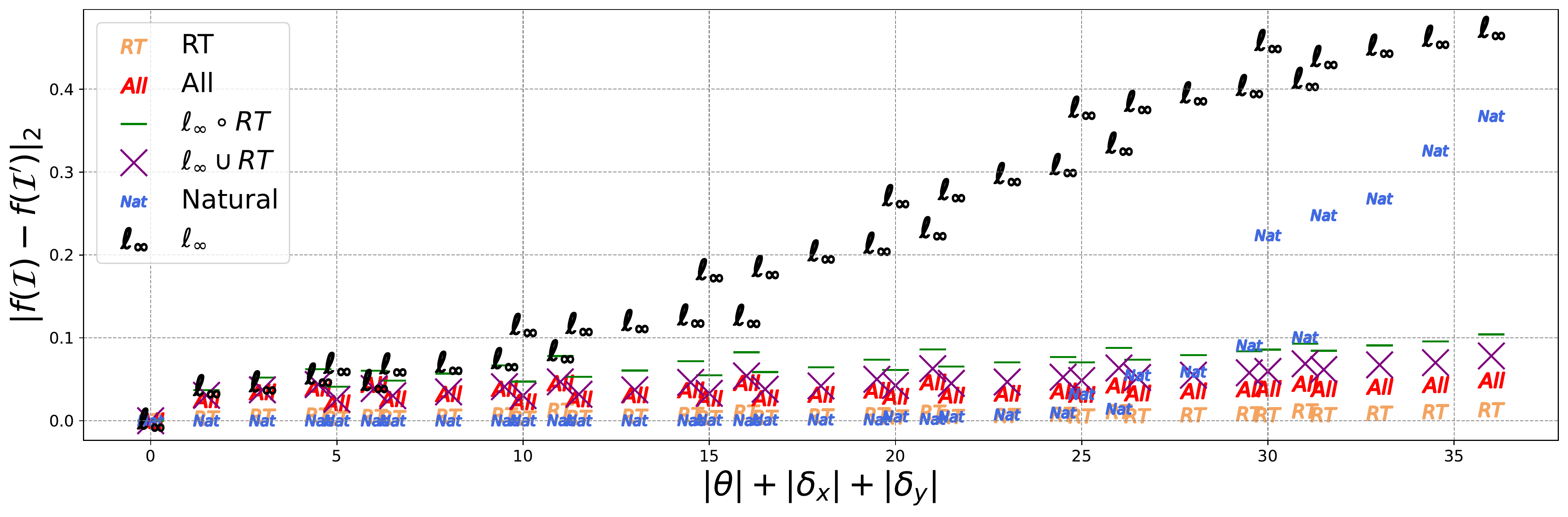}
    \caption{\textbf{$\text{TRADES}_{\text{All}}$ learns strong RT stability.} The figure plots the median logit difference between natural and RT-transformed images for different classifiers over a batch of 128 CIFAR-10 images for different RT transformation sizes. We note that the strongest natural classifiers (RT \& Natural) have a median Lipschitz constant near $0$ for smaller transformation sizes, while other models converge to a larger median Lipschitz constant.}
    \label{fig:logits-rt}
    \vspace{-10pt}
\end{figure}

\begin{figure}[ht]
    \centering
    \includegraphics[width=\linewidth]{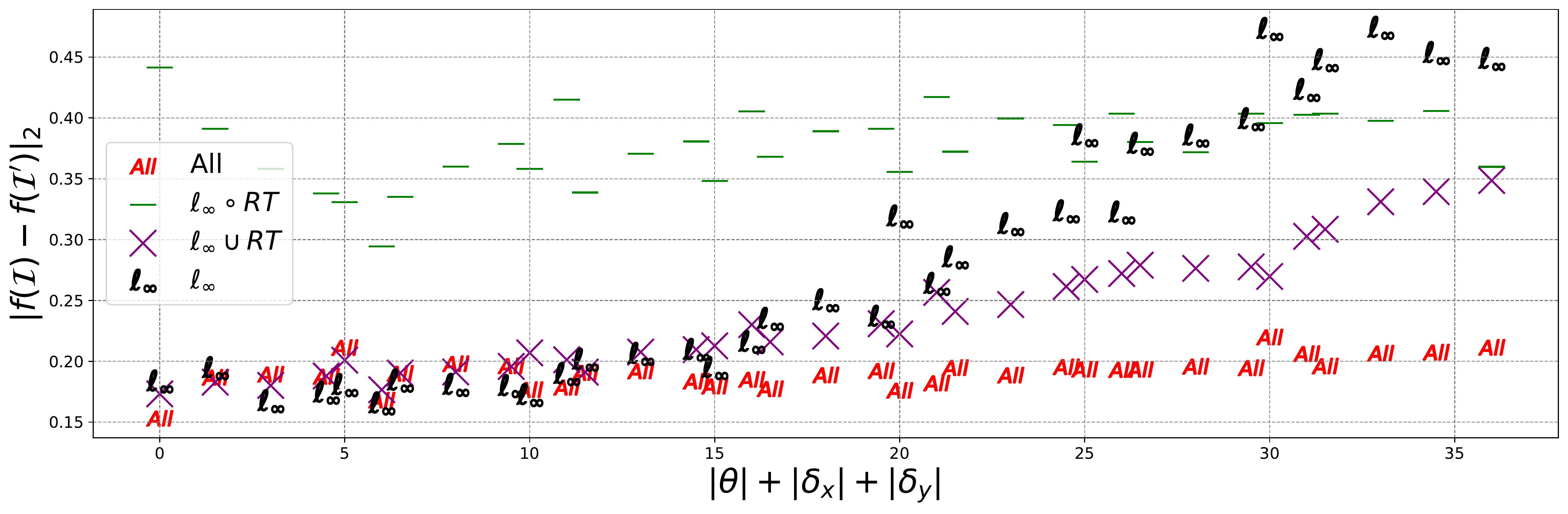}
    \caption{\textbf{$\text{TRADES}_{\text{All}}$ learns similar logit stability across a range of RT transformations.}The figure plots the median logit difference between natural and $\ell_\infty$-perturbed RT-transformed images (\textit{i.e.}, composite adversarial examples) for different classifiers over a batch of 128 CIFAR-10 images for different RT transformation sizes. The $x$-axis measures the strength of the RT transformation applied. We note that the natural and RT classifiers are not stable to these perturbations (see Figure~\ref{fig:logits-linf-rt}), while the other models garner more stability. $\text{TRADES}_{\text{All}}$ remains the most stable of all. }
    \label{fig:logits-linf-rt-no-nat}
    \vspace{-10pt}
\end{figure}

\begin{figure}[ht]
    \centering
    \includegraphics[width=\linewidth]{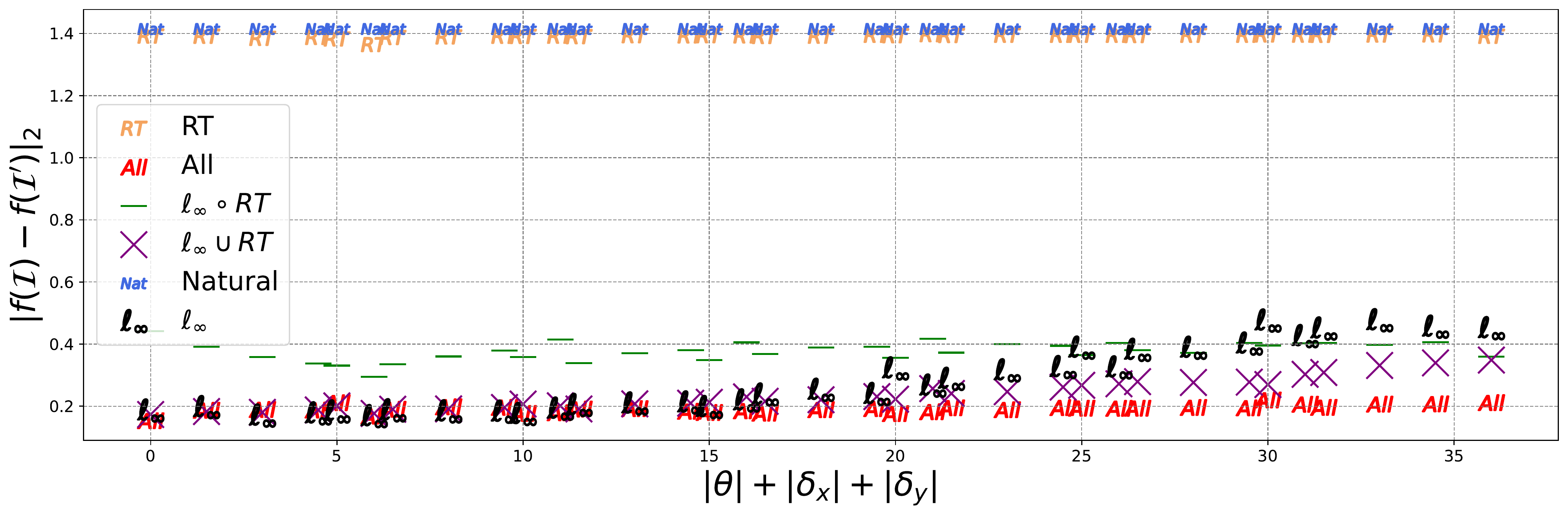}
    \caption{\textbf{Same as Figure~\ref{fig:logits-linf-rt-no-nat} but including RT and natural.} }
    \label{fig:logits-linf-rt}
    \vspace{-10pt}
\end{figure}

\clearpage
% ---------------------------------------------------------------------------
\section{The \texorpdfstring{$\ell_p$}{composition} setting is insufficient}
\label{sec:app-setting}
% ---------------------------------------------------------------------------

\begin{figure}[ht]
    \centering
    \includegraphics[width=\linewidth]{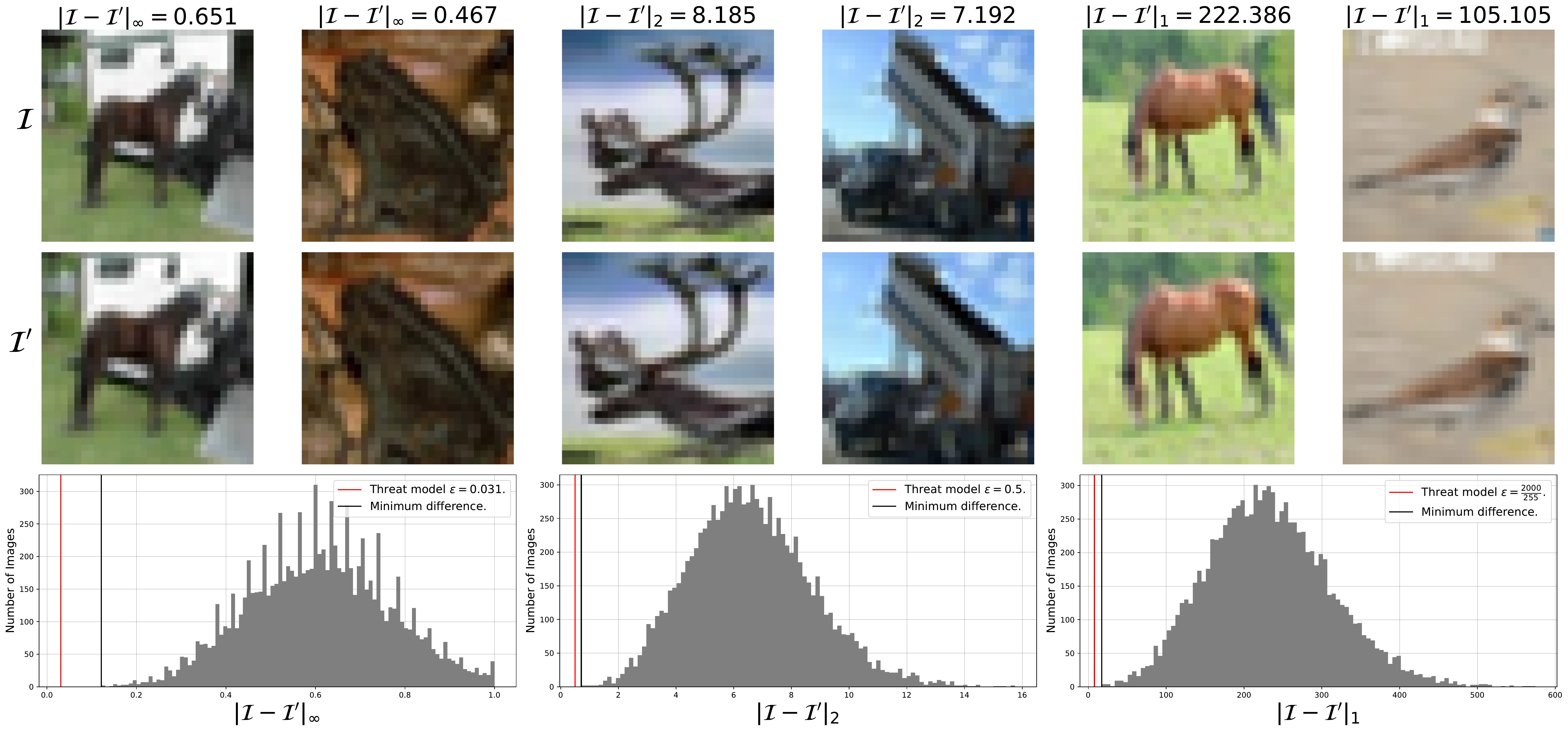}
    \caption{\textbf{$\ell_\infty$ norm threat models do not include other transformations of interest.} The first two rows show pairs of clean images ($\Ic$) from the CIFAR-10 test set and these same images scaled, rotated, and translated using bilinear interpolation ($\Ic'$). The three plots show the frequency distribution of $\|\Ic-\Ic'\|_p$ for $p \in \{1,2,\infty\}$ over the entire CIFAR-10 test set.}
    \label{fig:spatial-lp}
\end{figure}
Any classifier deemed robust should not be any more vulnerable to the affine transformed images seen in Figure~\ref{fig:spatial-lp}, simply because these are ill-defined in the $\ell_p$-threat model. To further demonstrate this, we plot affine transformed images for the entire CIFAR-10 test set, showing that no pair of natural and affine transformed images is considered valid under the $\ell_p$ threat models, despite being very perceptually similar to a human and certainly semantically equivalent. This highlights the need to consider threat models which go beyond the $\ell_p$ setting.

\clearpage
% ---------------------------------------------------------------------------
\section{Proof of Theorem \ref{theorem: main}}
\label{sec:app-theo}
% ---------------------------------------------------------------------------

\newcommand{\vxh}{\hat{\vx}}
\newcommand{\be}{\boldsymbol{\epsilon}_{\hat{f}_u}}
\newcommand{\efu}{\epsilon_{\hat{f}_u}}
\newcommand{\bef}{\boldsymbol{\epsilon}_{\hat{f}}}
\newcommand{\ef}{\epsilon_{\hat{f}}}

\begin{align}
\hspace{-4px}Y \stackrel{\text{u.a.r.}}{\sim} \{-1, 1\}, \; X_0 | Y = y := \begin{cases}
			y, & \text{w.p. } p;\\
            -y, & \text{w.p. } 1-p,
		 \end{cases} \; X_t | Y = y \sim \mathcal{N}\left(y\eta, 1\right),\;\; 1 \leq t\leq d-1. %=1,2,...,d. 
\end{align}

\begin{theorem}[Theorem \ref{theorem: main} Restated]
Given data distribution $\mathcal{D}$ where $p \geq \frac{1}{2}$ , $\eta \geq \frac{1}{\sqrt{d}}$, and $d \geq 24$, no linear classifier $f:\Rb^{d} \rightarrow \{ -1, 1 \}$, where $f(\vx) = \sign(\vw^T \vx)$, can obtain robust accuracy $> 0.5$ under the $\ell_{\infty}\circ \text{RT}$ threat model with $\ell_{\infty}$ budget $\epsilon=2\eta$ and RT budget $N = \frac{d}{8}$.
\end{theorem}

\begin{proof}

Suppose for a contradiction that there exists a classifier $\hat{f}:\Rb^{d} \rightarrow \{ -1, 1 \}$ that can obtain $\ell_{\infty}\circ RT$ robustness $> 0.5$. Moreover, we suppose w.l.o.g. that $\hat{f}$ is the optimal compositionally robust linear classifier. That is, for all linear classifiers $\tilde{f}$, $\min_{\boldsymbol{\epsilon}, \text{RT}} \text{Pr}[\tilde{f}(\text{RT}(\vx) + \boldsymbol{\epsilon}) = y] \leq \min_{\boldsymbol{\epsilon}, \text{RT}} \text{Pr}[\hat{f}(\text{RT}(\vx) + \boldsymbol{\epsilon}) = y]$. 

Let $\hat{f}(\vx) = \sign(\vw_{\hat{f}}^T \vx)$. We first prove a lemma that shows some structure on $\vw$. The key insight is that any linear classifier that is non-trivially robust to a compositional adversary must rely on the strongly-correlated feature $x_0$. Moreover, placing excessive weight on some weakly-correlated feature $x_i$ can only increase the probability of a misclassification under a compositional adversary, as an $\ell_{\infty}$ perturbation $-\epsilon y \sign(w_i)$ can flip the weakly correlated feature's sign so it becomes weakly correlated with $-y$. We define $\mathcal{R} \subseteq [d]$ by the set of indices of the features that the strongly-correlated feature can move to via an RT transformation. Here, $|\mathcal{R}| = N$.

\begin{lemma}
\label{lemma:1}
$\hat{f}$ must have 0 weight over all features $x_i$, where $i \in [d] \setminus \mathcal{R}$.
\end{lemma}

\begin{proof}
     Suppose that $\hat{f}$ contains a non-zero weight $w_{\hat{f}, i}$, where $i \in [d] \setminus \mathcal{R}$. Now, consider the modified classifier $\hat{f}_0$, where we set $w_{\hat{f}_0, i} := 0$ and leave all other weights unchanged. Since the strongly correlated feature cannot move to index $i$ via any RT transformation, then applying any RT transformation will leave the distribution of $x_i$ unchanged. Since the difference between the classification of $\hat{f}$ and $\hat{f}_0$ depends solely on $x_i$, it suffices to compare the probability of correct classification between $\hat{f}$ and $\hat{f}_0$ under an $\ell_{\infty}$ adversary.
    
    Let $p_{\hat{f}} := \min_{\boldsymbol{\epsilon}} \text{Pr}[\hat{f}(\vx + \boldsymbol{\epsilon}) = y]$, and define $p_{\hat{f}_0}$ similarly. Moreover, let $\boldsymbol{\epsilon}_{\hat{f}} := \argmin_{\boldsymbol{\epsilon}} \text{Pr}[\hat{f}(\vx + \boldsymbol{\epsilon}) = y]$, and define $\boldsymbol{\epsilon}_{\hat{f}_0}$ similarly. Now, since the classification decision of $\hat{f}_0$ does not depend on $x_i$, then by applying a perturbation $\hat{\epsilon} := -\epsilon y \sign(w_{\hat{f}, i})$ to $x_i$, where $\epsilon = 2 \eta$, we have that:
    \begin{align*}
        p_{\hat{f}_0}
        &= \text{Pr}[\hat{f}_0(\vx + \boldsymbol{\epsilon}_{\hat{f}_0}) = y]
        \\&= \text{Pr}[y\sign(\vw_{\hat{f}_0}^T (\vx + \boldsymbol{\epsilon}_{\hat{f}_0})) > 0]
        \\&= \text{Pr}\left[y\sign(\sum_{j \neq i} w^T_{\hat{f}_0, j} (x_j + \epsilon_{\hat{f}_0, j}))  > 0\right]
        \\&> \text{Pr}\left[y\sign(\sum_{j \neq i} w^T_{\hat{f}_0, j} (x_j + \epsilon_{\hat{f}_0, j}) + w_{\hat{f}, i}(x_i + \hat{\epsilon})) > 0\right]\\&\qquad\text{(since $\text{Pr}[ |w_{\hat{f}, i}| \mathcal{N}(-\eta, 1) > 0] < 0.5$.)}
        \\&\geq \text{Pr}[y\sign(\vw_{\hat{f}}^T (\vx + \boldsymbol{\epsilon}_{\hat{f}})) > 0]\\&\qquad\text{(by the optimality of $\boldsymbol{\epsilon}_{\hat{f}}$.)}
        \\&= p_{\hat{f}},
    \end{align*}
    where the strict inequality follows since $x_i\sim\mathcal{N}(\eta y, 1)$ and applying perturbation $\hat{\epsilon} = -\epsilon y \sign(w_{\hat{f}, i})$, where $\epsilon = 2 \eta$, will change the distribution of $x_i$ so that $x_i\sim\mathcal{N}(-\eta, 1)$. However, this contradicts the optimality of $\hat{f}$, completing the proof.

\end{proof}

\begin{lemma}
\label{lemma:2}
$\hat{f}$ does not contain any zero weights over the features $x_i$, where $i \in \mathcal{R}$.
\end{lemma}

\begin{proof}
Suppose that $w_{\hat{f}, i} = 0$ for some $i \in \mathcal{R}$. Note that by Lemma \ref{lemma:1} we know that $w_{\hat{f}, j} = 0$ for all $j \in [d] \setminus \mathcal{R}$. Therefore, we know that there must exist some non-zero weight $w_{\hat{f}, n}$ for $n \in \mathcal{R}$, as otherwise the classifier $\hat{f}$ would reduce to the $0$ classifier, which attains accuracy $0.5$. Now, we define a compositional adversary $\hat{\text{RT}}$ that first swaps $x_0$ with $x_i$, and then applies an $\ell_{\infty}$ perturbation $\boldsymbol{\hat{\epsilon}}$ defined by $\hat{\epsilon}_l := -\epsilon y \sign(w_{\hat{f}, l})$ to all features $x_l$, $l \in [d]$. Note that by swapping $x_0$ with $x_i$, the strongly correlated feature is zeroed out. Moreover, the remaining weakly correlated features are correlated with $-y$ by the $\ell_{\infty}$ adversarial perturbation. Letting $\mathcal{P}$ denote the set of indices whose weights of $\hat{f}$ are nonzero, we can see that for any $\vx$:
\begin{align*}
    \text{Pr}\Big[\hat{f}(\hat{\text{RT}}(\vx) + \hat{\boldsymbol{\epsilon}}) = y\Big]
    &= \text{Pr}\Big[y\sign(\vw_{\hat{f}}^T (\hat{\text{RT}}(\vx) + \boldsymbol{\hat{\epsilon}})) > 0\Big]
    \\&= \text{Pr}\Big[y \sum_{i \in \mathcal{P}} |w_{\hat{f}, i}|\mathcal{N}(- \eta y, 1) > 0\Big]
    \\&= \text{Pr}\Big[\sum_{i \in \mathcal{P}} |w_{\hat{f}, i}|\mathcal{N}(- \eta, 1) > 0\Big]
    \\&< 0.5, \quad\text{(since $-\eta < 0$.)}
\end{align*}
contradicting that $\hat{f}$ attains nontrivial robustness against a composite adversary, which completes the proof.
\end{proof}

\begin{lemma}
\label{lemma:3}
$\hat{f}$ must have uniform weights over the features $x_i$, where $i \in \mathcal{R}$.
\end{lemma}

\begin{proof}
% (We first need to justify that $\vw_{\hat{f}}$ contains at least two non-zero weights $w_{\hat{f}, i}, w_{\hat{f}, j}$, where $i \neq j$ and $i,j \in \mathcal{R}$).
Note that this is trivially true if $|\Rc|=1$. Let us consider the case where $|\Rc| \geq 2$. Suppose that $\exists\;\; w_{\hat{f},i} \neq w_{\hat{f},j}$, for some $i,j \in \Rc, i \neq j$. By Lemma \ref{lemma:2}, we know that for all $k \in \mathcal{R}$, $w_{\hat{f}, k} \neq 0$. We assume w.l.o.g. that $|w_{\hat{f}, i}| \leq |w_{\hat{f}, k}| \; \forall k \in \Rc $.  Moreover, by Lemma \ref{lemma:1} we know that $w_{\hat{f}, l} = 0$ for all $l \in [d] \setminus \mathcal{R}$. Consider the modified classifier $\hat{f}_{u}$, where we set $w_{\hat{f}_u, k} := w_{\hat{f}, i}$ for all $k \in \mathcal{R}$. Note that since the weights in $\hat{f}_{u}$ over the features in $\mathcal{R}$ are uniform, then an RT adversary has no power under this classifier. We will show that a composite adversary can always exploit $\hat{f}$ to a greater extent than $\hat{f}_u$ by placing the strongly correlated feature at position $i$ and applying an $\ell_\infty$ perturbation to flip the distribution of features $k \in \Rc \setminus \{ i \}$.

Let $\vxh$ denote the RT perturbed vector where feature $x_0$ has been swapped with $x_i$. Let $p_{\hat{f}} := \min_{\boldsymbol{\epsilon}, \text{RT}} \text{Pr}[\hat{f}(\text{RT}(\vx) + \boldsymbol{\epsilon}) = y]$, and define $p_{\hat{f}_u}$ similarly. Moreover, let $\boldsymbol{\epsilon}_{\hat{f}}, \text{RT}_{\hat{f}} := \argmin_{\boldsymbol{\epsilon}, \text{RT}} \text{Pr}[\hat{f}(\text{RT}(\vx) + \boldsymbol{\epsilon}) = y]$, and define $\boldsymbol{\epsilon}_{\hat{f}_u}, \text{RT}_{\hat{f}_u}$ similarly. Lastly, define $\hat{\boldsymbol{\epsilon}}$ by $\hat{\epsilon}_l := -\epsilon y \sign(w_{\hat{f}, i})$, for all $l \in [d]$, and define $\tilde{\boldsymbol{\epsilon}}$ by $\tilde{\epsilon}_l := -\epsilon y \sign(w_{\hat{f}, l})$, for all $l \in [d]$.

\begingroup
\allowdisplaybreaks
\begin{align*}
    p_{\hat{f}_u}
    &= \text{Pr}[\hat{f}_u(\text{RT}_{\hat{f}_u}(\vx) + \be) = y]
    \\&= \text{Pr}[\hat{f}_u(\vx + \be) = y] \quad\text{(since the weights over features in $\mathcal{R}$ are uniform.)}
    \\&= \text{Pr}\left[y\sign\left(\vw_{\hat{f}_u}^T (\vx + \be)\right) > 0\right]
    \\&= \text{Pr}\left[y\sign\left(\sum_{j \in \mathcal{R}} w_{\hat{f},i} (x_j + \epsilon_{\hat{f}_u, j}) \right) > 0\right]
    \\&= \text{Pr}\left[y\sign\left(\sum_{j \in \mathcal{R}} w_{\hat{f},i} (x_j + \hat{\epsilon}_j) \right) > 0\right] \quad\text{(as $\be = \hat{\boldsymbol{\epsilon}}$ is optimal.)}
    \\&= \text{Pr}\left[y\sign\left(\sum_{j \in \mathcal{R} \setminus \{0\}} w_{\hat{f},i} (x_j -\hat{\epsilon}_j y \sign(w_{\hat{f}, i})) + w_{\hat{f},i} (x_0 -\hat{\epsilon}_i y \sign(w_{\hat{f}, i})) \right) > 0\right]
    \\&= \text{Pr}\left[y\sign\left(\sum_{j \in \mathcal{R} \setminus \{i\}} w_{\hat{f},i} (\hat{x}_j -\hat{\epsilon}_j y \sign(w_{\hat{f}, i})) + w_{\hat{f},i} (\hat{x}_i -\hat{\epsilon}_i y \sign(w_{\hat{f}, i})) \right) > 0\right] \\&\qquad\text{(as $x_0 = \hat{x}_i$, $x_i = \hat{x}_0$, and $x_j = \hat{x}_j$ for $j \in \mathcal{R} \setminus \{ 0, i \}$.)}
    \\&> \text{Pr}\left[y\sign\left(\sum_{j \in \mathcal{R} \setminus \{i\}} w_{\hat{f},j} (\hat{x}_j -\tilde{\epsilon}_j y \sign(w_{\hat{f}, j})) + w_{\hat{f},i} (\hat{x}_i -\tilde{\epsilon}_i y \sign(w_{\hat{f}, i})) \right) > 0\right] \\&\qquad\text{(as $\text{Pr}\Big[\mathcal{N}(-\eta, 1) > 0\Big] < 0.5$ and $|w_{\hat{f},i}| \leq |w_{\hat{f},j}|,\ \forall j$.)}
    \\&= \text{Pr}\Big[\hat{f}(\hat{\vx} + \tilde{\boldsymbol{\epsilon}}) = y\Big]
    \\&\geq \text{Pr}\Big[\hat{f}(\text{RT}_{\hat{f}}(\vx) + \boldsymbol{\epsilon}_{\hat{f}}) = y\Big] \quad\text{(by the optimality of $\boldsymbol{\epsilon}_{\hat{f}}$ for $\text{RT}_{\hat{f}}$.)} 
    \\&= p_{\hat{f}}.
\end{align*}
\endgroup
    
However, this contradicts the optimality of $\hat{f}$, completing the proof.
\end{proof}

Lemma \ref{lemma:1}, \ref{lemma:2}, and \ref{lemma:3} suggest the following structure on $\vw_{\hat{f}}$: 
\begin{align*}
    w_{\hat{f},i}=\begin{cases}c, & \text{if }\; i \in \Rc;\\ 0,&\text{otherwise,}\end{cases}
\end{align*} 
for some $c > 0$. Note that trivially $c \not< 0$, as otherwise the classification decision would be correlated with $-y$, which would attain natural accuracy $< 0.5$. Thus, $\hat{f}(\vx) = \sign(\vw_{\hat{f}}^T \vx) = c\sign(\sum_{i \in \Rc} x_i)$ for some $c >0$.

% \clearpage
This suggested structure of the classifier reveals an interesting insight on the relationship between the strongly and weakly correlated features. Namely, as $N = |\mathcal{R}|$ grows, the effect of the strongly correlated feature will be increasingly diluted by the cumulative effect of the weakly correlated features with indices in $\mathcal{R}$, since the weights over the features with indices in $\mathcal{R}$ are uniform. Since $\hat{f}$'s weights are uniform over $\Rc$, its classification decision will be invariant to the position of the strongly correlated feature. Therefore, it is sufficient to consider its robustness to a worst case $\ell_\infty$ perturbation. For simplicity, we will assume that $\Rc = [|\Rc|]$ i.e., that the first $|\Rc|$ features are RT reachable and that the strongly correlated feature is at $x_0$. Let $\boldsymbol{\epsilon}_{\hat{f}} := \argmin_{\boldsymbol{\epsilon}} \text{Pr}[\hat{f}(\vx + \boldsymbol{\epsilon}) = y]$. We note that applying $\epsilon_{\hat{f},0}=-\epsilon y\sign(w_{\hat{f},0})=-2\eta y\sign(w_{\hat{f},0})$ to the strongly correlated feature will bring the example closest to the decision boundary.

\begingroup
\allowdisplaybreaks
\begin{align}
    \text{Pr} [\hat{f}(\vx + \bef) = y]=& \text{Pr} [y \vw_{\hat{f}}^T(\vx + \bef) > 0] 
    \nonumber\\=& \text{Pr} \left[yc \sum_{i = 0}^{|\Rc|-1} x_i+\epsilon_{\hat{f},i} > 0\right] 
    \nonumber\\=& \text{Pr} \left[yc \left(\sum_{i = 1}^{|\Rc|-1} x_i+\epsilon_{\hat{f},i}\right) + yc(x_0+\epsilon_{\hat{f},0}) > 0\right]
     \nonumber\\=& p\cdot \text{Pr} \left[yc \left(\sum_{i = 1}^{|\Rc|-1} \Nc(\eta y,1)+\epsilon_{\hat{f},i}\right) + yc(y+\epsilon_{\hat{f},0})> 0\right] 
     \nonumber\\ &+ (1-p)\cdot \text{Pr} \left[yc \left(\sum_{i = 1}^{|\Rc|-1} \Nc(\eta y,1)+\epsilon_{\hat{f},i}\right) + yc(-y+\epsilon_{\hat{f},0}) > 0\right]
     \nonumber\\=& p\cdot \text{Pr} \left[yc \left(\sum_{i = 1}^{|\Rc|-1} \Nc(-\eta y,1)\right) + c+yc\epsilon_{\hat{f},0} > 0\right] 
     \nonumber\\ &+ (1-p)\cdot \text{Pr} \left[yc \left(\sum_{i = 1}^{|\Rc|-1} \Nc(-\eta y,1)\right)-c+yc\epsilon_{\hat{f},0} > 0\right]
     \nonumber\\=& p\cdot \text{Pr} \left[yc  \Nc(-\eta y(|\Rc|-1),(|\Rc|-1)) + c+yc\epsilon_{\hat{f},0} > 0\right] 
     \nonumber\\ &+ (1-p)\cdot \text{Pr} \left[yc  \Nc(-\eta y(|\Rc|-1),(|\Rc|-1)) - c+yc\epsilon_{\hat{f},0} > 0\right]
     \nonumber\\=& p\cdot \text{Pr} \left[\Nc\left(-c\eta (|\Rc|-1),c^2(|\Rc|-1)\right)+ c+yc\epsilon_{\hat{f},0} > 0\right] 
     \nonumber\\ &+ (1-p)\cdot \text{Pr} \left[ \Nc\left(-c\eta (|\Rc|-1),c^2(|\Rc|-1)\right) - c+yc\epsilon_{\hat{f},0} > 0\right]
     \nonumber\\=& p\cdot \text{Pr} \left[\Nc\left(-c\eta (|\Rc|-1)+ c+yc\epsilon_{\hat{f},0},c^2(|\Rc|-1)\right) > 0\right] 
     \nonumber\\ &+ (1-p)\cdot \text{Pr} \left[ \Nc\left(-c\eta (|\Rc|-1) - c+yc\epsilon_{\hat{f},0},c^2(|\Rc|-1)\right) > 0\right]
     \nonumber\\=& p\cdot \text{Pr} \left[\Nc\left(-c\eta (|\Rc|-1)+ c+yc(-2\eta y),c^2(|\Rc|-1)\right) > 0\right] 
     \nonumber\\ &+ (1-p)\cdot \text{Pr} \left[ \Nc\left(-c\eta (|\Rc|-1) - c+yc(-2\eta y),c^2(|\Rc|-1)\right) > 0\right]
     \nonumber\\=& p\cdot \text{Pr} \left[\Nc\left(c(-\eta|\Rc|+\eta+1-2\eta),c^2(|\Rc|-1)\right) > 0\right] 
     \nonumber\\ &+ (1-p)\cdot \text{Pr} \left[ \Nc\left(c(-\eta|\Rc|+\eta-1-2\eta),c^2(|\Rc|-1)\right) > 0\right].\label{eq:proba}
\end{align}
\endgroup

To upper bound the probability of a correct classification (to complete the proof by contradiction) using an interpretable quantity, we will express the two probabilities above in terms of the standard normal CDF, $\Phi(\cdot)$. The scalars $\alpha_1, \alpha_{-1} \in \Rb$ in equations \ref{eq:a} and \ref{eq:b} represent distances from 0 in units of standard deviation $\sigma=c\sqrt{(|\Rc|-1)}$ to the means of their respective normal distributions in \eqref{eq:proba}.

\begin{align}
c(-\eta|\Rc|+\eta+1-2\eta)-\alpha_1\cdot c\sqrt{(|\Rc|-1)} = 0\label{eq:a},\\
c(-\eta|\Rc|+\eta-1-2\eta)-\alpha_{-1}\cdot c\sqrt{(|\Rc|-1)} = 0\label{eq:b}.
\end{align}
Solving for $\alpha_1$ in \eqref{eq:a},
\begin{align*}
&\qquad c(-\eta|\Rc|-\eta+1-2\eta)-\alpha_1\cdot c\sqrt{(|\Rc|-1)}= 0\Longleftrightarrow \alpha_1 = \frac{\eta(-|\Rc|-3)+1}{\sqrt{(|\Rc|-1)}}.\\
% &\Longleftrightarrow \alpha_1\cdot \sqrt{d-1} = (d-1)\eta+1\\
% &\Longleftrightarrow \alpha_1 = \frac{(d-1)\eta+1}{\sqrt{d-1}}.
\end{align*}

Similarly we obtain $\alpha_{-1} = \frac{\eta(-|\Rc|-3)-1}{\sqrt{(|\Rc|-1)}}$. By symmetry of the normal distribution, we can rewrite $\text{Pr}[f(\vx) = y]= \text{Pr}\left[y \vw^T\vx > 0 \right]$ as follows:
\begin{align}
    \text{Pr}\left[y \vw^T\vx > 0 \right]
    &= p\cdot\Phi(\alpha_1)+(1-p)\cdot\Phi(\alpha_{-1})
    \nonumber\\&= p\cdot\Phi\left(\frac{\eta(-|\Rc|-3)+1}{\sqrt{(|\Rc|-1)}}\right)+(1-p)\cdot\Phi\left(\frac{\eta(-|\Rc|-3)-1}{\sqrt{(|\Rc|-1)}}\right).\label{eq:finalform}
\end{align}
Note that the inputs to the cumulative distribution function decrease as $\eta$ grows. Therefore, if suffices to upper-bound the probability of correct classification with $\eta=\frac{1}{\sqrt{d}}$. Plugging in $\eta = \frac{1}{\sqrt{d}}$, we obtain:  
\begin{align*}
    \text{Pr}\left[y \vw^T\vx > 0 \right]= p\Phi\left(\frac{-|\Rc|-3}{\sqrt{d(|\Rc|-1)}}+\frac{1}{\sqrt{(|\Rc|-1)}}\right)+(1-p)\Phi\left(\frac{-|\Rc|-3}{\sqrt{d(|\Rc|-1)}}-\frac{1}{\sqrt{(|\Rc|-1)}}\right).
\end{align*}
This expression depends on $|\Rc|$ and $d$. But recall that we have assumed that the size of $N = |\Rc| = \frac{d}{8}$. Therefore, we have that:
\begin{align*}
    \text{Pr}\left[y \vw^T\vx > 0 \right]= p\cdot\Phi\left(\frac{-\frac{d}{8}-3}{\sqrt{d(\frac{d}{8}-1)}}+\frac{1}{\sqrt{(\frac{d}{8}-1)}}\right)+(1-p)\cdot\Phi\left(\frac{-\frac{d}{8}-3}{\sqrt{d(\frac{d}{8}-1)}}-\frac{1}{\sqrt{(\frac{d}{8}-1)}}\right).
\end{align*}

To complete our proof by contradiction, it suffices to show that the quantity $\frac{-\frac{d}{8}-3}{\sqrt{d(\frac{d}{8}-1)}}+\frac{1}{\sqrt{(\frac{d}{8}-1)}} \leq 0$ for $d \geq 24$, since  $\frac{-\frac{d}{8}-3}{\sqrt{d(\frac{d}{8}-1)}}-\frac{1}{\sqrt{(\frac{d}{8}-1)}} \leq \frac{-\frac{d}{8}-3}{\sqrt{d(\frac{d}{8}-1)}}+\frac{1}{\sqrt{(\frac{d}{8}-1)}}$, $p \cdot \Phi(0) + (1-p) \cdot \Phi(0) = \Phi(0) = 0.5$, and $\Phi(x)$ is a monotonically increasing function of $x$.

First, we observe that:
\begin{align*}
    \frac{-\frac{d}{8}-3}{\sqrt{d(\frac{d}{8}-1)}}+\frac{1}{\sqrt{(\frac{d}{8}-1)}}
    &= \left(\frac{-(\frac{d}{8} + 3)}{\sqrt{d}} + 1\right)\frac{1}{\sqrt{(\frac{d}{8}-1)}}.
\end{align*}
Thus, it suffices to show that $\frac{(\frac{d}{8} + 3)}{\sqrt{d}} \geq 1$ for all $d \geq 24$. First, observe that this is the case when $d = 24$:
\begin{align*}
    \frac{\frac{24}{8} + 3}{\sqrt{24}}
    = \frac{6}{\sqrt{24}}
    \geq 1.
\end{align*}

Further, when $d \geq 24$, we can see that $\frac{(\frac{d}{8} + 3)}{\sqrt{d}}$ is a monotonically increasing function of $d$ by treating $d$ as a continuous variable and looking at the first derivative of $g(d) = \frac{(\frac{d}{8} + 3)}{\sqrt{d}}$:
\begin{align*}
    g'(d)
    = \frac{1}{2}\frac{d^{\frac{-1}{2}}}{8} - \frac{3}{2}d^{\frac{-3}{2}}
    = \frac{d - 24}{16\sqrt{d^3}},
\end{align*}
and thus $g'(d) \geq 0$ exactly when $d \geq 24$, as desired. Therefore, $\frac{(\frac{d}{8} + 3)}{\sqrt{d}} \geq 1$ for all $d \geq 24$. Therefore, the classification accuracy of $\hat{f}$ is upper bounded by 0.5 under the compositional adversary, which contradicts our assumption that $\hat{f}$ achieves nontrivial compositional robustness. This concludes the proof.
\end{proof}

\end{document}

%% file: math_commands.tex
\usepackage{amsmath,amsfonts,bm}

\def\eqref#1{equation~\ref{#1}}
\def\1{\bm{1}}

\def\vw{{\bm{w}}}
\def\vx{{\bm{x}}}

\DeclareMathAlphabet{\mathsfit}{\encodingdefault}{\sfdefault}{m}{sl}
\SetMathAlphabet{\mathsfit}{bold}{\encodingdefault}{\sfdefault}{bx}{n}

\DeclareMathOperator*{\argmin}{arg\,min}

\DeclareMathOperator{\sign}{sign}

%% file: bens_math_commands.tex
\newtheorem{theorem}{Theorem}[section]

\newtheorem{lemma}[theorem]{Lemma}

\newcommand{\Ic}{\mathcal{I}}

\newcommand{\Nc}{\mathcal{N}}

\newcommand{\Rc}{\mathcal{R}}
\newcommand{\Rb}{\mathbb{R}}